\documentclass{article}
\usepackage{palatino,epsfig,latexsym,natbib}
\usepackage{amsfonts}
\usepackage{xspace}                           
\usepackage{booktabs}                                     \usepackage{algorithm,algorithmic,dsfont}
\usepackage{tikz}
\usepackage{balance}
\usepackage{amsmath}
 \usepackage{amsthm}

\newcommand{\ignore}[1]{}
\newtheorem{theorem}             {Theorem}
\newtheorem{lemma}      [theorem]{Lemma}

\usetikzlibrary{arrows,patterns,plotmarks,shapes,snakes,er,3d,automata,backgrounds,topaths,trees,petri,mindmap}

\newcommand{\oneplusone}{(1+1)~EA\xspace}

\newcommand{\prob}[1]{\Pr\left(#1\right)}
\newcommand{\expect}[1]{\mathbf{E}\left[#1\right]}

\DeclareMathOperator{\poly}{poly}

\DeclareMathOperator{\unif}{Unif}
\DeclareMathOperator{\Pois}{Pois}

\newcommand{\ab}{\hspace{0.125em}}                        
\newcommand{\ie}{\hbox{i.\ab e.}\xspace}

\begin{document}

\title{\bf A Parameterized Complexity Analysis of Bi-level Optimisation with Evolutionary Algorithms}

\author{
Dogan Corus\\
psxdc1@nottingham.ac.uk\\
School of Computer Science\\ 
The University of Nottingham, UK
\and 
Per Kristian Lehre\\
perkristian.lehre@nottingham.ac.uk\\
School of Computer Science\\ 
The University of Nottingham, UK
\and 
Frank Neumann\\
frank.neumann@adelaide.edu.au\\
Optimisation and Logistics\\
School of Computer Science\\
The University of Adelaide, Australia
\and 
Mojgan Pourhassan\\
mojgan.pourhassan@adelaide.edu.au\\
Optimisation and Logistics\\
School of Computer Science\\
The University of Adelaide, Australia
}

\maketitle

\begin{abstract}
  Bi-level optimisation problems have gained increasing interest in
  the field of combinatorial optimisation in recent years.  With this
  paper, we start the runtime analysis of evolutionary algorithms for
  bi-level optimisation problems.  We examine two NP-hard problems,  
  the generalised minimum spanning tree problem (GMST), and the
  generalised travelling salesman problem (GTSP) in the context of
  parameterised complexity.

  For the generalised minimum spanning tree problem, we analyse the
  two approaches presented by \cite{HuR12} with respect
  to the number of clusters that distinguish each other by the chosen
  representation of possible solutions.  Our results show that a
  (1+1)~EA working with the spanning nodes representation is not a
  fixed-parameter evolutionary algorithm for the problem, whereas the
  global structure representation enables to solve the problem in
  fixed-parameter time. We present hard instances for each approach
  and show that the two approaches are highly complementary by proving
  that they solve each other's hard instances very efficiently.

  For the generalised travelling salesman problem, we analyse the
  problem with respect to the number of clusters in the problem
  instance. Our results show that a (1+1)~EA working with the global
  structure representation is a fixed-parameter evolutionary algorithm
  for the problem.
\end{abstract}


\section{Introduction}

Many interesting combinatorial optimisation problems are hard to solve
and meta-heuristic approaches such as local search, simulated
annealing, evolutionary algorithms, and ant colony optimisation have
been used for a wide range of these problems.

In recent years, researchers became very interested in bi-level
optimisation for single-objective~\citep{Koh07,LegillonLT12} and
multi-objec\-tive
problems~\citep{DBLP:conf/emo/DebS09,DBLP:journals/ec/DebS10}. Such
problems can be split up into an upper and a lower level problem which
depend on each other. By fixing a possible solution for the upper
level problem, the lower level is optimised with respect to the given
objective and the constraints imposed by the choice of the upper
level.

Recently, Hu and Raidl~\citep{HuR11,HuR12} have proposed two different
approaches for the generalised minimum spanning tree problem (GMSTP).
Both approaches work with an upper layer and a lower layer
solution. The upper layer solution $x$ is evolved by an evolutionary
algorithm whereas the optimal solution $y$ of the lower layer problem
corresponding to a particular search point $x$ of the upper layer can
be found in polynomial time using deterministic algorithms.

 Our goal is to understand the two different
approaches by parameterised computational complexity
analysis~\citep{Downey1999}. The computational complexity analysis of
meta-heuristics plays a major role in the theoretical analysis of this
type of algorithms and studies the runtime behaviour with respect to
the size of the given input. We refer the reader
to~\citep{BookAugDoe,Neumann2010} for a comprehensive presentation.
Parameterised complexity analysis takes into account the runtime of
algorithms in dependence of an additional parameter which measures the
hardness of a given instance. This allows us to understand which
parameters of a given NP-hard optimization problem make it hard or
easy to be optimised by heuristic search methods.  In the context of
evolutionary algorithms, the term fixed-parameter evolutionary
algorithms has been defined in \citep{KratschNeumannGECCO09}. An
evolutionary algorithm is called a fixed-parameter evolutionary
algorithm for a given parameter $k$ iff its expected runtime is
bounded by $f(k) \cdot \poly(n)$ where $f(k)$ with respect to the
input size $n$. Parameterised computational complexity analysis of
evolutionary algorithms have been carried out for the vertex cover
problem~\citep{KratschNeumannGECCO09}, the computation of maximum leaf
spanning trees~\citep{DBLP:conf/ppsn/KratschLNO10}, makespan
scheduling~\citep{Sutton2012makespan}, and the travelling salesperson
problem~\citep{SuttonN12}.

We push forward the parameterised analysis of evolutionary algorithms
and present the first analysis in the context of bi-level
optimization.  In our investigations, we take into account the two NP-hard problems 
the generalised minimum spanning tree problem (GMSTP) and the generalised
travelling salesman problem (GTSP) which share the parameter, number of clusters $m$.
We consider two different bi-level representations
for GMTSP which both have a polynomially solvable lower level part. For the
\emph{Spanning Nodes Representation}, we present worst case examples
which show that there are instances leading to an optimization time of
$\Omega(n^m)$. For the \emph{Global Structure
  Representation}, we show that it leads to a fixed-parameter
evolutionary algorithm with respect to the number of clusters
$m$. Furthermore, we present an instance class where the algorithm
using the \emph{Global Structure Representation} encounters an
optimization time of $m^{\Omega(m)}$.  Analysing both approaches on
each others worst-case instances, we show that they solve them very
efficiently. This shows the complementary abilities of these two
representations for the GMSTP. Then we extend our results for
\emph{Global Structure Representation} to GTSP to show that a similar algorithm
has an expected optimisation time of $m^{\Omega(m)}$ for this problem as well.

The paper is divided into two main parts according to the two different problems. 
The first part (based on the conference version~\citep{Corus2013GMST})  where the GMSTP problem is investigated is presented in Section~\ref{sec:mst}.
We show hard instances for the \emph{Spanning Nodes Representation} in
Section~\ref{sec:span} and show that a simple evolutionary
algorithms needs exponential time even if the number of clusters is
small. In Section~\ref{sec:global}, we examine the \emph{Global
Structure Representation} and show that this leads to
fixed-parameter evolutionary algorithms for GMSTP. We point out complementary
abilities in Section~\ref{sec:complement}.
This article extends the conference version~\citep{Corus2013GMST} by
investigations of the GTSP and some generalizations. We examine the GTSP problem with the corresponding \emph{Global Structure Representation} in Section~\ref{sec:tsp} and provide upper and lower bounds on the optimisation time of the considered algorithm. Furthermore, we point out in Section~\ref{sec:gen} general characteristics  which allows this fixed-parameter
result to be extended to other problems.

\section{Generalised Minimum Spanning Tree Problem}
In this section, we consider the GMSTP problem and provide the runtime analysis with respect to bi-level representations given in~\citep{HuR11,HuR12}.
\label{sec:mst}

\subsection{Preliminaries}
\label{sec:prelim}

We consider the generalised minimum spanning tree problem (GMSTP)
introduced in~\citep{MyungLT95}.  The input is given by an undirected
complete graph $G= (V,E,c)$ on $n$ nodes with a cost function $c\colon
E \rightarrow \mathds{R}_+$ that assigns positive costs to the edges.
Furthermore, a partitioning of the node set $V$ into $m$ pairwise
disjoint clusters $V_{1}, V_{2}, \ldots, V_{m}$ is given such that 
$n=\sum_{i=1}^{m}\left|{V_i}\right|$.

A solution to the GMSTP problem consists of two components, the $m$
chosen nodes $P$, called \emph{the spanning nodes}, in the $m$
clusters, and a minimum spanning tree $T$ on the graph induced by the
spanned nodes. More precisely, a solution $S=(P,T)$ consists of a node
set $P=(p_1, \ldots, p_m) \in V^m$, where $V^m =V_1\times
V_2\times\cdots\times V_m$ and a spanning tree $T \subseteq E$ on the
subgraph $G[P] = G(P, \{e\in E \mid e \subseteq P \})$ induced by
$P$. The cost of $T$ is the cost of the edges in $T$, \ie,
\begin{align*}
 C(T)=\sum_{(u,v) \in T} c(u,v).
\end{align*}

The goal is to compute a solution $S^*=(P^*,T^*)$ which has minimal
cost among all possible solutions $S=(P,T)$.  For an easier
presentation, we assume in some cases that edge costs can be
$\infty$. In this case, we restrict our investigations to solutions
that do not include edges with cost $\infty$. Alternatively, one might
view this as the GMSTP defined on a graph that is not necessarily
complete.

The GMSTP problem is NP-hard~\citep{MyungLT95} and two different
bi-level evolutionary approaches have been examined in~\citep{HuR12}.
The first approach presented in~\citep{HuR12} uses the \emph{Spanned
  Nodes Representation}. It selects in the upper level problem a node
for each cluster and computes on the lower level a minimum spanning
tree (using for example Kruskal's algorithm in time $O(m \log m)$) on
the induced subgraph.

The second approach uses the \emph{Global Structure Representation}.
It constructs a complete graph $H=(V',E')$ from the given input graph
$G=(V,E,c)$ and the set of pair-wise disjoint clusters $V_{1}, V_{2},
\ldots, V_{m}$. The node $v_i \in V'$, $1 \leq i \leq m$, corresponds
to the cluster $V_i$ in $G$. The search space for the upper level
consists of all spanning trees of $H$ and the spanned nodes of the
different clusters are selected in time $O(n^2)$ using the dynamic
programming approach of Pop~\citep{Pop04}.

For our theoretical investigations, we measure the runtime of the
algorithms by the number of fitness evaluations required to obtain an
optimal solution. We call this the \emph{optimization time} of the examined algorithm.
The \emph{expected optimization time} refers to the
expected number of fitness evaluations until an optimal solution has
been obtained for the first time.

\subsection{Spanned Nodes Representation} 
\label{sec:span}
We analyse the cluster based \oneplusone in this section. Our first
theorem shows that this algorithm is an
XP-algo\-rithm~\citep{Downey1999}, \ie an algorithm that runs in time
$O(n^{g(m)})$ where $g(m)$ is a computable function only depending on
$m$, when choosing the number of clusters $m$ as a parameter.

\begin{theorem}\label{thm:cl-upper}
  For any instance of the GMSTP problem,
  the expected  time until the cluster based \oneplusone reaches the
  optimal solution is $O(n^m)$.
\end{theorem}

\begin{proof}  
  For any search point $x$, let $w(x)\in[m]$ denote the number of
  clusters where the spanned node representation includes a suboptimal
  node. If the algorithm chooses all $w(x)$ suboptimal clusters for
  mutation and selects the optimal node in each of them, then the
  optimal solution is obtained. Since $w(x)\leq m$, the probability
  that all suboptimal clusters are mutated in a single step is at
  least $(1/m)^m$. The probability of choosing the optimal node in
  cluster $i$ is $1/|V_i|$. Thus, the probability of jumping to the
  optimal solution from any search point is at least
  $$(m)^{-m} \prod_{i=1}^{m}|V_i|^{-1}.$$

  Since $\sum_{i=1}^{m}V_i = n$, it holds that
  \begin{align*}
    \prod_{i=1}^{m}\frac{1}{|V_i|}  \geq (m/n)^m.
  \end{align*}
  Therefore, the probability of reaching the optimal solution in one
  step is $\Omega(n^{-m})$, and the expected time to reach the optimal
  solution is bounded from above by $O(n^m)$.
\end{proof}

We now consider an instance of GMSTP which is difficult for the
cluster based \oneplusone. The hard instance $G_{S}$ for the
\emph{Spanning Nodes Representation} is illustrated in
Figure~\ref{fig:spanned-instance}. It consists of $m$ clusters, where
one cluster is called the \emph{central} cluster, and the $m-1$ other
clusters are called \emph{peripheral} clusters.  Each cluster contains
$n/m$ nodes and we assume that $n=m^2$ holds. The nodes in the
peripheral clusters are called peripheral nodes, and the nodes in the
central cluster are called central nodes.  Within each cluster, one of
the nodes is called \emph{optimal}, and is marked black in the
figure. The remaining $(n/m)-1$ nodes are called \emph{sub-optimal}
nodes, and are marked white in the figure. The instance is a
bi-partite graph, where edges connect peripheral nodes to central
nodes. The cost of any edge between two optimal nodes is 1, the cost
of any edge between two suboptimal nodes is 2.  The cost of any edge
between a suboptimal peripheral node and the optimal central node is
$n^2$, and the cost of any edge between an optimal peripheral node and
a suboptimal central node is $n$. A cluster is called \emph{optimal}
in a solution, if the solution has chosen the optimal node in that
cluster.

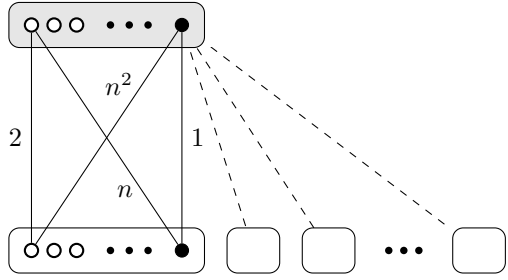
\begin{figure}[t]
  \centering
  \begin{tikzpicture}[node distance=0.5cm,
  dot/.style={fill=black,circle,minimum size=5pt}
  ]





  \draw[rounded corners] (-0.3,-0.3) rectangle (2.3,0.3);
  \node[place,fill=black,circle,minimum size=5pt]       (p1opt)     at (2.0,.0) {};
  \node[place,fill=white,circle,minimum size=5pt,thick] (p1nonopt3) at (0.0,0.0) {};
  \node[place,fill=white,circle,minimum size=5pt,thick]             at (0.3,0.0) {};
  \node[place,fill=white,circle,minimum size=5pt,thick]             at (0.6,0.0) {};

  \draw[rounded corners] (5.6,-0.3) rectangle node (p4) {} (6.3,0.3);  

  \fill[black] (5.15,0.0) circle (0.3ex);
  \fill[black] (4.95,0.0) circle (0.3ex);
  \fill[black] (4.75,0.0) circle (0.3ex);

  \draw[rounded corners] (3.6,-0.3) rectangle node (p3) {} (4.3,0.3);  
  \draw[rounded corners] (2.6,-0.3) rectangle node (p2) {} (3.3,0.3);



  \draw[fill=black!10,rounded corners] (-0.3,2.7) rectangle (2.3,3.3);
  \node[place,fill=black,circle,minimum size=5pt]       (copt)     at (2.0,3.0) {};
  \node[place,fill=white,circle,minimum size=5pt,thick] (cnonopt3) at (0.0,3.0) {};
  \node[place,fill=white,circle,minimum size=5pt,thick]             at (0.3,3.0) {};
  \node[place,fill=white,circle,minimum size=5pt,thick]             at (0.6,3.0) {};

  \fill[black] (1.55,0.0) circle (0.3ex);
  \fill[black] (1.30,0.0) circle (0.3ex);
  \fill[black] (1.05,0.0) circle (0.3ex);

  \fill[black] (1.55,3.0) circle (0.3ex);
  \fill[black] (1.30,3.0) circle (0.3ex);
  \fill[black] (1.05,3.0) circle (0.3ex);
  
  \path[draw] (p1opt)     -- node[right] {$1$} (copt);
  \path[draw] (p1opt)     -- node[near start, left] {$n$} (cnonopt3);
  \path[draw] (p1nonopt3) -- node[near end, left] {$n^2$} (copt);
  \path[draw] (p1nonopt3) -- node[left] {$2$} (cnonopt3);

  \begin{scope}
    \clip (-0.3,0.3) rectangle (6.3, 2.7); 
    \draw[dashed] (p2) -- (copt);
    \draw[dashed] (p3) -- (copt);
    \draw[dashed] (p4) -- (copt);
  \end{scope}

\end{tikzpicture}

  \caption{Hard instance $G_{S}$ for \emph{Spanning Node Representation}.}
  \label{fig:spanned-instance}
  \label{fig:spanned}
\end{figure}

\begin{theorem}\label{thm:cl-lower}
  Starting with an initial solution chosen uniformly at random, the
  expected optimization time of the cluster based (1+1)~EA on $G_{S}$
  is $\Omega(n^m)$.

  Furthermore, for any constant $\epsilon>0$, the probability of
  having obtained an optimal solution after at most
  $n^{(1-\epsilon)m}$ iterations is $e^{-\Omega(m)}$.
\end{theorem}
\begin{proof}
  We define two phases for the run of the \oneplusone. The first phase
  consists of the first $n-1$ iterations while the second phase starts
  at the end of the first phase and continues for $n^{m/12}$
  iterations.  Four distinct events are considered failures during the
  run of the \oneplusone for the instance described above.
  \begin{enumerate}
  \item The first failure occurs if during the first phase of the run,
    the algorithm obtains a search point with less than $m/6$
    sub-optimal peripheral clusters.
  \item The second type of failure occurs when the central cluster
    fails to switch to a suboptimal node at least once during the
    first phase.
  \item The third type of \emph{failure} occurs when the algorithm
    does not switch all the optimal peripheral clusters to suboptimal
    clusters during the second phase.
  \item The fourth failure corresponds to a direct jump to the optimal
    solution during the second phase.
\end{enumerate}	
We first show that the probability of the first failure event is at
most $\exp(-\frac{m}{12})$. This implies that with overwhelmingly high
probability, a constant fraction of peripheral clusters is always
suboptimal during the first $n-1$ iterations.
For $i\in[m-1]$ and $t\geq 0$, let $Z_i(t), $ be a random variable
such that $Z_i(t)=1$ if cluster $V_i$ is always sub-optimal in
iteration 0 through iteration $t$, and $Z_i(t)=0$ otherwise. The
probability that a suboptimal node is selected in the initial solution
is $1-m/n$.  In the following iterations, the probability that a
cluster is selected for mutation and that its new spanned node is
optimal is $(1/m)(m/n)=1/n$. So it is clear that
$$\prob{Z_i(t) = 1} \geq (1-m/n)(1-1/n)^t.$$  
By linearity of expectation, $$\expect{\sum_{i=1}^{m-1}Z_i(t)} \geq
(m-1) \left(1-\frac{m}{n} \right)\left(1-\frac{1}{n}\right)^t.$$
Considering a phase length of $t=n-1$, and assuming that $m$ is
sufficiently large and $n=m^2$ holds, we
get $$\expect{\sum_{i=1}^{m-1}Z_i(t)}\geq \frac{m}{3}.$$ Finally, a
Chernoff bound \citep{motwani:randomized} implies that
\begin{align*}
  \prob{\sum_{i=1}^{m-1}Z_i(t)
    \leq \left(1-\frac{1}{2}\right)\frac{m}{3}} 
    \leq \exp\left(-m/12\right).
\end{align*}
  
We then show that the probability of the second failure event is
$\exp(-\Omega(\sqrt{n}))$.  In each iteration the probability to
switch the central cluster to a suboptimal node is at least
  $$p=\frac{1}{m} \left(1-\frac{m}{n} \right)=\Omega \left(\frac{1}{\sqrt{n}} \right).$$ The
probability that this event does not occur in $n-1$ steps is
\begin{align*}
    (1-p)^{n-1} & = \left((1-p)^{1/p}\right)^{(n-1)p}\\
               & \leq \exp(-p(n-1)) = \exp(-\Omega(\sqrt{n})).
\end{align*}

Now, we show that the probability of the third failure event is less
than $n^{-m/12}$, assuming that the first two failure events do not
occur.  As long as the central cluster remains suboptimal, switching a
suboptimal node in a peripheral cluster to an optimal node will result
in an extra cost of $n-2$.  Conversely, switching an optimal
peripheral cluster into a sub-optimal cluster will decrease the cost
by $n-2$. As long as there is at least one suboptimal peripheral
cluster, making the central cluster optimal will incur an extra cost
of at least $n^2-2$. So, during phase two, the algorithm can not make
any suboptimal cluster optimal unless all suboptimal clusters are made
optimal in the same iteration. The probability of making at least
$m/6$ suboptimal peripheral clusters optimal simultaneously is at most
$$\left(\frac{1}{m} \cdot \frac{m}{n} \right)^{m/6} = \left(\frac{1}{n} \right)^{m/6}.$$

Since the probability to jump to the optimal solution is at most
$n^{-m/6}$ in each iteration, it holds by the union bound that the
probability of failure event three is at
most $$n^{-m/6}n^{m/12}=n^{-m/12}.$$

Finally, we show that the probability of failure event four is
$O(n^{-m/13})$.  The probability that an optimal peripheral cluster is
made suboptimal by the \oneplusone is at least
$$\frac{1}{m} \cdot
\frac{n-m}{n} \cdot \left(1- \frac{1}{m}\right)^{m-1} \geq \frac{1}{3m}.
$$ 
The expected time $E[T^{sub}]$ until all peripheral clusters have
become suboptimal is therefore at most $m \cdot 3m =
O(m^2)$. Considering a phase of length $n^{m/12}$ and taking into
account $m^2 =n$, it holds by Markov's inequality that the probability
of a type four failure is
\begin{align*}
  \prob{T^{sub}> n^{m/12} } \leq O(m^2)n^{-m/12} =O({n^{-m/12 + 4}}).
\end{align*}

By union bound, the probability that any type of failure occurs is
less than the sum of their independent probabilities, which is
$e^{-\Omega(m)}$. Hence, with overwhelmingly high probability, after
the second phase, the algorithm has obtained a locally optimal
solution where all peripheral clusters are sub-optimal.  After that
iteration, the probability to jump directly to the optimal solution is
$n^{-m}$, and the expected time for this event to occur is $n^{m}$.

Let $E$ be the event that no failure occurs. Then, the first statement
of the theorem follows by the law of the total probability,
\begin{align*}
\expect{T}&\geq \expect{T|E}\prob{E}\\
          & =   \Omega(n^m) (1- e^{-\Omega(m)}) \\
          & =   \Omega(n^m).
\end{align*}

Furthermore, by union bound, it holds that
\begin{align*}
  \prob{T< n^{(1-\epsilon)m}\mid E} \leq n^{(1-\epsilon)m}n^{-m}=n^{-\epsilon m}.
\end{align*}
Hence, the second statement of the theorem follows by the law of
total probability
\begin{align*}
\prob{T< n^{(1-\epsilon)m}} 
   & = \prob{T< n^{(1-\epsilon)m}\mid E}\prob{E}  \\
   & \quad + \prob{T< n^{(1-\epsilon)m}\mid \overline{E}}\prob{\overline{ E}} \\
   & \leq \prob{T< n^{(1-\epsilon)m}\mid E} + \prob{ \overline{E}}\\
   & \leq n^{-\epsilon m} + e^{-\Omega(m)} = e^{-\Omega(m)}.
\end{align*}
\end{proof}

Our results for the \emph{Spanned Nodes Representation} show that the
cluster based \oneplusone obtains an optimal solution in time $O(n^m)$
and our analysis for the hard instance $G_{S}$ shows that this bound
is tight.

\subsection{Global Structure Representation}
\label{sec:global}

The second approach examined in \citep{HuR12} uses the \emph{Global
  Structure Representation}.  It works on the complete graph 
$H = (V', E')$ obtained from the input graph $G=(V,E, c)$.  The node 
$v_i \in V'$, $1 \leq i \leq m$, represents the cluster $V_i$ of $G$.

The upper level solution in the \emph{Global Structure Representation}
is a spanning tree $T$ of $H$ and the lower level solution is a set of
nodes $P=(p_1,\ldots, p_m)$ with $p_i \in V_i$ that minimises the cost
of a spanning tree which connects the clusters in the same way as
$T$. Given a spanning tree $T$ of $H$, the set of nodes $P$ can be
computed in time $O(n^{2})$ using dynamic programming~\citep{Pop04}.

We consider the tree based \oneplusone outlined in
Algorithm~\ref{alg:tree}. It starts with a spanning tree $T$ of $H$
that is chosen uniformly at random. In each iteration, a new solution
$T'$ of the upper layer is obtained by performing $K$ edge-swaps to
$T$. Here the parameter $K$ is chosen according to the Poisson
distribution with expectation $1$. In one edge swap, an edge $e$
currently not present in the solution is introduced and an edge from
the resulting cycle is removed such that a new spanning tree of $H$ is
obtained.  After having produced the offspring $T'$, the corresponding
set of nodes $P'$ is computed using dynamic programming. $P$ and $T$
are replaced by $P'$ and $T'$ if the cost of the new solution is not
worse than the cost of the old one.

\begin{algorithm}[t]
  \caption{Tree based \oneplusone}
  \begin{algorithmic}
    \STATE Choose a spanning tree $T$ of $H$.
    \STATE Apply dynamic programming to find the minimum spanned
    nodes $P=(p_1,\ldots, p_m)$ induced by $T$.
    \WHILE{termination condition not satisfied}
    \STATE {$T'\leftarrow T$ }
    \FOR{$i \in \left[ K \right]$ where $K\sim \Pois(1)$}
    \STATE{ Sample edge $e\sim \unif\left(E' \setminus T\right)$ }
    \STATE{ Sample edge $e'\sim \unif$(edges in cycle in $T'\cup\left\{ e \right\}$) }
    \STATE{ $T'\leftarrow T'\cup\{e\}\setminus \left\{ e'\right\}$ }
    \ENDFOR
    \STATE{Apply dynamic programming to find a set of spanned nodes $P'=(p'_1,\ldots,p'_m)$  with respect to $T'$ of minimal cost.}
    \IF{$\sum_{(i,j)\in T'}c(p'_i,p'_j)\leq \sum_{(i,j)\in T}c(p_i,p_j)$}
    \STATE {$P\leftarrow P'$}
    \STATE {$T\leftarrow T'$}
    \ENDIF
    \ENDWHILE
  \end{algorithmic}
\label{alg:tree}
\end{algorithm}

In the following, we show that the tree based \oneplusone is a
fixed-parameter evolutionary algorithm for the GMSTP problem when
considering the number of clusters $m$ as the parameter. We do this by
transferring the result of~\citep{Pop04} to the tree based \oneplusone.

\begin{theorem}
  The expected time of the tree based \oneplusone to find the optimal
  solution for any instance of the GMSTP problem is
  $O(m^{3(m-1)})$. Furthermore, for any $k\geq 1$, the
  probability that an optimal solution is not found within $ekm^{3(m-1)}$
  steps is less than $\exp(-k)$.
\end{theorem}
\begin{proof}
  An upper layer solution is a tree $T$ of $H$.  Let $T^*$ be any tree
  of $H$ for which there exists a set $P^*$ of spanning nodes such
  that $T^*$ and $P^*$ form an optimal solution. For any non-optimal
  solution $T$, define $w((T)$ as the number of edges in $T^*$ which
  are missing in $T$.

  The mutation operator can convert a non-optimal solution $T$ into
  the optimal solution $T^*$ with a sequence of $w((T)\leq m-1$ edge
  exchange operations. The probability that the mutation operator
  exchanges $w((t)\leq m-1$ edges in one mutation step is at least
  $$\prob{\Pois(1)=m-1}=1/e(m-1)!.$$

  In each exchange operation, if there are $i$ optimal edges missing,
  then the probability that one of the missing optimal edges is
  inserted is at least $i/m^2$. After the addition of an optimal edge,
  the probability of excluding a non-optimal edge is at least $1/m$
  since the largest cycle cannot be longer than $m$.  At most $m-1$
  non-optimal edges must be exchanged in this manner.  So the
  probability that the non-optimal solution $T$ will be converted to
  the optimal solution $T^*$ in one mutation step is at least
  \begin{align*}
    \frac{1}{e(m-1)!}\cdot  
      \prod_{i=1}^{m-1} \frac{i}{m^2}\cdot\frac{1}{m} 
      \geq (1/e) m^{-3(m-1)}.
  \end{align*}

  So, the expected time to achieve an optimal solution is in
  $O(m^{3(m-1)})$. Furthermore, the probability that the optimal
  solution has not been created after $ekm^{3(m-1)}$ iterations is
  $$(1-(1/e)m^{-3(m-1)})^{ekm^{3(m-1)}}\leq \exp(-k).$$
\end{proof}

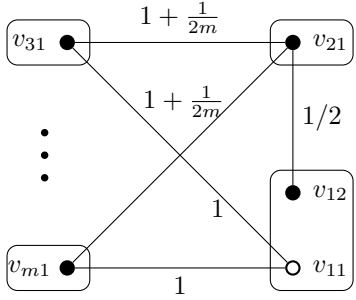
\begin{figure}
  \centering
  \begin{tikzpicture}[node distance=0.5cm,
  dot/.style={fill=black,circle,minimum size=5pt}]







  \node[place,fill=black,circle,minimum size=5pt,thick] (v12) at (3.0,1.0) {};
  \node[place,fill=white,circle,minimum size=5pt,thick] (v11) at (3.0,0.0) {};
  
  \draw[rounded corners] (2.7,-0.3) rectangle (3.8,1.3);
  \draw[rounded corners] (2.7, 2.7) rectangle (3.8,3.3);

  \node[place,fill=black,circle,minimum size=5pt,thick] (v21) at (3.0,3.0) {};

  \node[right of=v11] {$v_{11}$};
  \node[right of=v12] {$v_{12}$};
  \node[right of=v21] {$v_{21}$};


  \draw[rounded corners] (-0.8,-0.3) rectangle (0.3,0.3);
  \draw[rounded corners] (-0.8, 2.7) rectangle (0.3,3.3);

  \fill[black] (-0.3,1.8) circle (0.3ex);
  \fill[black] (-0.3,1.5) circle (0.3ex);
  \fill[black] (-0.3,1.2) circle (0.3ex);


  \node[place,fill=black,circle,minimum size=5pt,thick] (p1) at (0.0,3.0) {};
  \node[place,fill=black,circle,minimum size=5pt,thick] (p2) at (0.0,0.0) {};

  \node[left of=p1] {$v_{31}$};
  \node[left of=p2] {$v_{m1}$};


  \path[draw] (v11) -- node[near start,left] {$1$} (p1);
  \path[draw] (v21) -- node[near start,left] {$1+\frac{1}{2m}$} (p2);

  \path[draw] (v11) -- node[below] {$1$} (p2);
  \path[draw] (v21) -- node[above] {$1+\frac{1}{2m}$} (p1);



  \path[draw] (v12) -- node [right] {$1/2$} (v21);

\end{tikzpicture}

  \caption{Hard instance $G_{G}$ for \emph{Global Structure Representation}.
           Edges not shown have weight $\infty$.\label{fig:global-new}}
\end{figure}

We now present an instance which is hard to be solved by the tree
based \oneplusone.  The instance $G_{G}$, illustrated in
Figure~\ref{fig:global-new}, consists of $n$ nodes and $m$
clusters. There are two central clusters denoted by $V_1$ and
$V_2$. The cluster $V_1$ contains the two nodes $v_{11}$ and
$v_{12}$. The remaining clusters $V_i, 2\leq i\leq m$, contain a
single node $v_{i1}$ each. The edges that connect the nodes $v_{11}$
to the peripheral cluster nodes have cost $1$. The edges that connect
$v_{21}$ to the peripheral clusters have weight $1+1/2m$.  The edge
that connects $v_{12}$ and $v_{21}$ have weight $1/2$. All other edges
have cost $\infty$. Hence, if the tree based \oneplusone connects
cluster $V_1$ and $V_2$, then the dynamic programming algorithm will
choose node $v_{12}$.

In our analysis, we will use the following lemma on basic properties
of the Poisson distribution with expectation $1$.

\begin{lemma}\label{lemma:poisson-tail}
If $K\sim\Pois(1)$, then $\prob{K\geq n}< 2(e/n)^n.$
\end{lemma}
\begin{proof}
Using Stirling's approximation of the factorial,
\begin{align*}
  n! > \sqrt{2\pi n} (n/e)^n > (n/e)^n.
\end{align*}
we obtain the simple bound
\begin{align*}
  \prob{K\geq n} 
  & = \sum_{i=n}^\infty \frac{1}{ei!}\\
  & < \sum_{i=n}^\infty \frac{1}{i!}\\
  & < \sum_{i=n}^\infty \frac{1}{n!}\left(\frac{1}{n+1}\right)^{i-n}\\
  & < (e/n)^n \sum_{i=0}^\infty \left(\frac{1}{n+1}\right)^{i}\\
  & = (e/n)^n \left(1+\frac{1}{n}\right).
\end{align*}
\end{proof}

Using the previous lemma, we are able to show that the tree based
\oneplusone finds it hard to optimize $G_{G}$ when choosing spanning
tree uniformly at random among all spanning trees having weight less
then $\infty$.

\begin{theorem}
  Starting with a spanning tree chosen uniformly at random among all
  spanning trees that have cost less than $\infty$, the expected optimization
  time of the tree based \oneplusone on $G_{G}$ is
  $\Omega((m/e)^{m-1})$.
\end{theorem}
\begin{proof} 
  Consider the instance in Figure \ref{fig:global-new}.  In the
  following, edge $e:=\{v_{12},v_{21}\}$ is the edge which connects
  the two central clusters.  The optimal solution corresponds
  to the spanning tree which includes edge $e$, and where all all
  other clusters are connected to cluster $V_2$. The solution where
  all peripheral clusters are connected to $V_1$, and where cluster
  $V_2$ is connected to one of the peripheral clusters, is a local
  optimum.
  
  We define four failure events that can occur during a run of the
  \oneplusone on this instance.  
  \begin{enumerate}
  \item The first type of failure occurs when the
    initial solution includes edge $e$. 
  \item The second type of failure occurs when
    less than $m/3$ of the peripheral clusters are connected to cluster
    $V_1$ in the initial solution.  
  \item The third type of failure occurs when the
    algorithm jumps directly to the optimal solution during the first
    $((m-2)/3e)^{(m-2)/6}$ iterations.
  \item  Finally, the fourth type of failure occurs if after
    iteration $((m-2)/3e)^{(m-2)/6}$, there exists a peripheral cluster which is not
    connected to cluster $V_1$.
  \end{enumerate}

  There are $m-2$ peripheral clusters which must be connected to
  either $V_1$ or $V_2$. Additionally, cluster $V_1$ and $V_2$ must be
  connected. This connection can be established either by adding edge
  $e=(v_{12}, v_{21})$, or by connecting a peripheral cluster to both
  $V_1$ and $V_2$. There are $2^{m-2}$ spanning trees which contain
  edge $e$, and $(m-2)\cdot 2^{m-3}$ spanning trees which do not
  contain edge $e$ since one of the $m-2$ peripheral clusters will be
  connected to both central clusters and the others will be connected
  to only one.
  So, the probability that a uniformly chosen spanning tree includes
  edge $e$ is $O(1/m)$, which is the probability of the first type of
  failure.

  Now, we show that the probability of the second type of failure is
  at most $\exp(-\Omega(m))$. Considering that the probability of a
  specific cluster is adjacent to $V_1$ in the initial solution is
  larger than $1/2$, the probability that less than $(m-2)/3$ clusters
  are connected to cluster $V_1$ in the initial solution is bounded by
  $\exp(-\Omega(m))$ using a Chernoff bound.
 
  Assuming that type one and type two failures did not occur, the
  algorithm cannot accept new search points where a cluster which is
  originally connected to $V_1$ is instead connected to $V_2$ since it
  will create an extra cost of $1/2m$. The only exception is if a type
  three failure occurs, \ie the algorithm jumps directly to the
  optimal solution where all the peripheral clusters are connected to
  $V_2$. For a type three failure to occur, at least $(m-2)/3$
  clusters have to be modified simultaneously.  Therefore, using
  Lemma~\ref{lemma:poisson-tail}, the probability of jumping directly
  to the optimal solution in a single step is bounded from above by
  $$2(3e/(m-2))^{(m-2)/3}.$$

  Taking a phase length of $((m-2)/3e)^{(m-2)/6}$ into account, the
  probability of a type three failure can be bounded from above using
  the union bound, as
  \begin{align*}
    ((m-2)/3e)^{(m-2)/6} 2(3e/(m-2))^{(m-2)/3} = (m/e)^{-\Omega(m)}.
  \end{align*}

  Now, it will be shown that the probability of a type four failure is
  $e^{-\Omega(m)}$. The probability that a single peripheral cluster
  which is connected to $V_2$ is switched to $V_1$ is bounded from
  below by
  \begin{align*}
    &\frac{1}{3} \cdot  \frac{1}{e} \cdot \frac{1}{(m^2-(m-1))}  =\Omega(1/m^{2}).
  \end{align*}

  Thus, the expected time between any such event is $O(m^2)$, and the
  expected time until all of the at most $m-2$ peripheral clusters are
  connected to $V_1$ is $\expect{T'}=O(m^3)$. By Markov's inequality,
  it holds for any nonnegative random variable $X$ that 
  $$\prob{X\geq k } \leq \frac{\expect{X}}{k}. $$ The probability that 
  it takes longer than $$k=((m-2)/3e)^{(m-2)/6}$$ iterations is
  therefore no more than $$\frac{\expect{T'}}{k}=O(m^3) \cdot
  ((m-2)/3e)^{-(m-2)/6} = (m/e)^{-\Omega(m)}.$$
 
  This proves our claim about the probability of failure event four.
  
  If none of the above mentioned failures occur, we reach the local
  optimum where all the peripheral clusters are connected to cluster
  $V_1$. From this point on, the probability to jump to the optimal
  solution is by Lemma \ref{lemma:poisson-tail} no more than
  $$
    2(e/(m-1))^{m-1}
  $$ 
  because it is necessary to make at least $m-1$ edge exchanges to
  reach the optimum. The expected time to reach the optimal solution
  conditional on no failure is therefore more than $(1/2)(m/e)^{m-1}$.

  Let $R$ be the event that no failure occurs. By the law of total
  probability, it follows that the expected time $\expect{T}$ to reach
  the global optimum is
  \begin{align*}
    \expect{T}&\geq \expect{T|R}\prob{R}\\
    & =  \Omega((m/e)^{m-1}) (1- O(1/m)) \\
    & = \Omega((m/e)^{m-1}).
  \end{align*}
\end{proof}

The previous theorem shows that there are instances for the cluster
based \oneplusone where the optimization time grows exponentially with
the number of clusters. In the next section, we will compare the two
different representations for GMSTP and show that they have
complementary capabilities.

\subsection{Complementary Abilities}
\label{sec:complement}

The two representations examined in the previous sections
significantly differ from each other. They both rely on the fact that
there is a deterministic algorithm which solves the lower level
problem in polynomial time. In this section, we want to examine the
differences between the two approaches.  We show that both
representations have complementary abilities and do this by examining
the algorithms on each others hard instance. Surprisingly, we find out
that the hard instance for one algorithm becomes easy to solve when
giving it as an input to the other algorithm.

In Section~\ref{sec:span}, we have shown a lower bound of
$\Omega(n^m)$ for the cluster based \oneplusone using the
\emph{Spanning Node Representation}.  The hard instance $G_{S}$ for
the cluster based \oneplusone given in Figure~\ref{fig:spanned}
consists of a central cluster to which all the other clusters are
connected. There are no other connections between the clusters. Hence,
there is only one spanning tree when working with the \emph{Global
  Structure Representation}. The dynamic programming algorithm that
runs on the lower layer of the tree based \oneplusone therefore solves
the problem in its first iteration.

The following theorem shows that these instances are easy to be
optimised by the tree based \oneplusone.
\begin{theorem}
  The tree based \oneplusone solves the instance $G_{S}$
  in expected constant time.
\end{theorem}
\begin{proof}
  There is only a single tree over the cluster graph. Hence, the
  algorithm selects the optimal tree in the initial iteration.
\end{proof}

For the tree based \oneplusone, working with the \emph{Global
  Structure Representation}, we showed that it finds the instance
$G_{G}$ given in Figure~\ref{fig:global-new} hard to solve.  Working
with the \emph{Spanning Nodes Representation}, there is only one
cluster that consists of two nodes where all the other clusters
contain exactly one node. Hence, an optimal solution is obtained by
computing a minimum spanning tree on the lower level if the right node
in the cluster of two nodes is chosen.  The following theorem
summarises this and shows that this instance become easy when working
with the cluster based \oneplusone.

\begin{theorem}
  The cluster based \oneplusone solves the instance $G_{G}$ in
  expected time $O(m)$.
\end{theorem}
\begin{proof}
  Cluster $V_1$ contains two nodes, and all other clusters contain a
  single node. If the initial solution is not already the optimal
  solution, the correct node of $V_1$ has to be selected using
  mutation. The node for the cluster $V_1$ is changed with probability
  $1/m$ and in such a step the correct node is selected with
  probability $1/2$. Hence, the probability of a mutation leading to
  an optimal solution is at least $\frac{1}{2m}$ and the expected
  waiting time for this event is $O(m)$.
\end{proof}

The investigations in this section show that the two examined
representations have complementary abilities. Switching from one
representation to the other one can significantly reduce the runtime.

\section{Generalised Travelling Salesman Problem}
\label{sec:tsp}
We now turn our attention to the NP-hard generalized traveling salesperson problem (GTSP).
Given a complete graph $G=(V,E, c)$ with a cost function $c \colon E \rightarrow \mathds{R}^+$ and a partitioning of the node set $V$ into $m$ clusters $V_i$, $1\leq i \leq m$, the goal is to find a cycle of minimal cost that contains exactly one node from each cluster.

The bi-level approach that we are studying is similar to the one
discussed in the previous section. We investigate the \emph{Global Structure Representation} which works on the complete graph 
$H = (V', E')$ obtained from the input graph $G=(V,E, c)$.  The node 
$v_i \in V'$, $1 \leq i \leq m$, represents the cluster $V_i$ of $G$.

The upper level solution in the \emph{Global Structure Representation}
is a Hamiltonian tour $\pi$ on $H$ and the lower level solution is a set of
nodes $P=(p_1,\ldots, p_m)$ with $p_i \in V_i$ that minimises the cost
of a Hamiltonian tour which connects the clusters in the same way as
$\pi$. Given the restriction imposed by the Hamiltonian tour $\pi$ of $H$, 
finding the optimal set of nodes $P$ can be done in time $O(n^3)$ by
using any shortest path algorithm. One such algorithm is
\textit{Cluster Optimisation} proposed initially by Fischetti et
al~\citep{ClusterOpt1997} and is widely used in the literature.
Let $\pi= (\pi_1, \ldots, \pi_m)$ be a permutation on the $m$ clusters and $p_i$ be the chosen node for cluster $V_{\pi_i}$, $1\leq i \leq m$. Then the cost of the tour $\pi$ is given by 

$$
c(\pi) = c(p_m,p_{1}) +\sum_{i=1}^{m-1} c(p_i,p_{i+1}).
$$


\begin{algorithm}[t]
  \caption{Tour-based \oneplusone}
  \begin{algorithmic}[1]
    \STATE Choose a random permutation $\pi$ (which is also a Hamiltonian tour) of the $m$ given clusters.
    \STATE Find the set of nodes P (one node in each cluster) to build the shortest path possible among those clusters with the given order, by means of any shortest path algorithm in time $O(n^3)$.
    \WHILE{termination condition not satisfied}
    \STATE {$\pi'\leftarrow \pi$ }
    \FOR{$i\in[K]$ where $K\sim 1+ Pois(1) $}
    \STATE Choose two nodes from $\pi'$ uniformly at random.
    \STATE{ $\pi'\leftarrow$ Perform the \emph{Jump} with the chosen nodes on $\pi'$}
    \ENDFOR
    \STATE{Find the set of nodes $P'=(p'_1, \ldots,p'_m)$ which minimizes the cost with respect to $\pi'$ in the lower level}
    \IF{$c(\pi') \leq c(\pi)$}
    \STATE {$P\leftarrow P'$}
    \STATE {$\pi\leftarrow \pi'$}
    \ENDIF
    \ENDWHILE
  \end{algorithmic}
\label{alg:tour}
\end{algorithm}

Our proposed algorithm starts with a random permutation of clusters which is always a Hamiltonian tour $\pi$, in a complete graph $H$.
In each iteration, a new solution
$\pi'$ of the upper layer is obtained by the commonly used \emph{Jump}
operator which picks a node and moves it to a random position in the permutation.  The number of jump operations carried out in a mutation step is chosen according to $1+ Pois(1)$, where $Pois(1)$ denotes the Poisson distribution with expectation $1$.  Although we are using the jump operator in these investigations, we would like to mention that similar results can be obtained for other popular mutation operators such as \textit{exchange} and \textit{inversion}.

\begin{theorem}
The expected optimization time of the tour based \oneplusone is
  $O(m!m^{2m})$.
\end{theorem}
\begin{proof}
We consider the probability of obtaining the optimal tour $\pi^*$ on the global graph $H$ from an arbitrary tour $\pi$. The number of \emph{Jump} operations required is at most $m$ (the number of clusters). The probability of picking the right node and moving it to the right position in each of those $m$ operations is at least $1/m^2$. 
We can obtain an optimal solution by carrying out a sequence of $m$ jump operations where the $i$th operation jumps element $\pi^*_i$ in $\pi$ to position $i$. 
Since the probability of $Pois(1)+1=m$ is  $1/(e(m-1)!)$, the probability of a specific sequence of $m$ \emph{Jump} operations to occur is bounded below by
 $$\frac{1}{e(m-1)!} \cdot \frac{1}{m^{2m}}.$$

Therefore, the expected waiting time for such a mutation is
 $$\left(\frac{1}{e(m-1)!} \cdot \frac{1}{m^{2m}} \right)^{-1}= O(m!m^{2m})$$
which proves the upper bound on the expected optimization time.
\end{proof}

 Note that this upper bound depends on the number of clusters. Since the computational effort required to assess the lower level problem is polynomial in input size, $O(n^3)$, this implies that the proposed algorithm is a fixed-parameter evolutionary algorithm for the GTSP problem and the parameter $m$, the number of clusters.

 So far we have found an upper bound for the expected time of finding
 an optimal solution using the presented algorithm. In this section we
 will find a lower bound for the optimization time. Figure
 \ref{fig:WorstGlobal} illustrates an instance of GTSP, $G_G$, for
 which finding the optimal solution is difficult by means of the
 presented bi-level evolutionary algorithm with Global Structure
 Representation. In this graph, each cluster has two nodes. On the
 upper layer a tour for clusters is found by the EA and on the lower
 layer the best node for that tour is found within each cluster. All
 white nodes (which represent sub-optimal nodes) are connected to each
 other, making any permutation of clusters a Hamiltonian tour even if
 the black nodes are not used. All such connections have a weight of
 $1$, except for those which are shown in the figure which have a
 weight of $2$. All edges between a black node and a white node and
 also all edges between black nodes have weight $m^2$, except the ones
 presented in the figure which have weight $1/m$.  An optimal solution of
 cost $1$ uses only edges of cost $1/m$ whereas local optimal
 solutions use only edges of cost $1$.  The tour comprising all black
 nodes in the same order as illustrated in
 Figure~\ref{fig:WorstGlobal} is the optimal solution.  Note that
 there are many local optimal solutions of cost $m$. For our analysis
 it is just important that they do not share any edge with an optimal
 solution.

\begin{figure}[t]
 \centering
 \includegraphics[width=200pt, height=200pt]{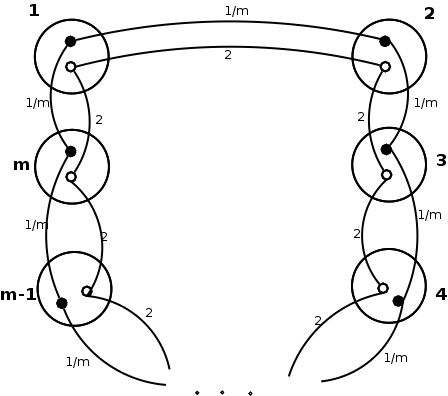}
 \caption{Hard instance $G_G$ for GTSP with \textit{global structure  representation}}
 \label{fig:WorstGlobal}
\end{figure}

The clusters are numbered in the figure, and a measure $S$ for
evaluating cluster orders is based on this numbering: Let $\pi=(\pi_1,
\ldots, \pi_m)$ represent the permutation of clusters in the upper
layer, then $S(\pi)=|\{i \mid \pi_{(i+1 \: \mod \: m)}=(\pi_i+1) \;
\mod \; m\}|$ indicates the similarity of the permutation
with the optimal permutation.  A large value of $S(\pi)$ means that
many clusters in $\pi$ are in the same order as
in the optimal solution.
Note that $S(\pi^*)=m$ for an optimal solution $\pi^*$.
A solution $\pi$ with $S(\pi)=0$ is locally optimal in the sense that
there is no strictly better solution in the neighbourhood induced by
the jump operator. The solutions with $S(\pi)=0$ form a plateau where all solutions differ from the optimal solution by $m$ edges.

We first introduce a lemma that will later help us with the proof of the lower bound on the optimization time.

\begin{lemma}
\label{lemma:localopt}
Let $\pi$ and $\pi'$ be two non-optimal cluster permutations for the
instance $G_G$. If $S(\pi')>S(\pi)$ then $c(\pi')>c(\pi)$.
\end{lemma}
\begin{proof}
In the given instance, all white nodes are connected to each other with a maximum weight of 2. These connections ensure that any permutation of the clusters, can result in a Hamiltonian tour with a cost of at most $2m$. Moreover, all connections between white nodes and black nodes have a weight of $m^2$. So the lower level will never choose a combination of white and black nodes because the cost will be more than $m^2$ while there is an option of selecting all white nodes with the cost of at most $2m$. On the other hand, for any permutation of clusters other than the Global Optimum, the lower level will not choose any black nodes, because it will not be possible to use all the $1/m$ edges and some $m^2$-weighted edges will be used again.
Let $a=S(\pi)$ be the number of clusters adjacent to each other correctly from the right side (having the same right-side neighbour as in the Global Optimum) in a solution $\pi$. Then $b=m-a$ is the number of clusters which have a different neighbour on their right. If $\pi$ is not the optimal solution, then the lower level will choose all white nodes. As a result, $a$ edges with weight 2 and $b$ edges with weight 1 will be used in that solution; therefore, the total cost of solution $\pi$ will be $c(\pi)=2a+b=2a+m-a= m+a$.
Consider a solution $\pi'$ with $a'=S(\pi')$ and $S(\pi') > S(\pi)$. We have $c(\pi') = m + a' > m+a = c(\pi)$ which completes the proof.
\end{proof}

Lemma~\ref{lemma:localopt} shows that any non-optimal offspring $\pi'$ of a solution $\pi$ is not accepted
if it is closer to an optimal solution $\pi^*$. This means that the algorithm finds it hard to obtain an optimal solution for $G_G$ and leads to an exponential lower bound on the optimization time as shown in the following theorem.

\begin{theorem}
\label{thm:GlobalWorst2}
Starting with a permutation of clusters chosen uniformly at random,  the optimisation time of the tour based \oneplusone on $G_G$ is $\Omega((\frac{m}{2})^\frac{m}{2})$ with probability $1-e^{-\Omega(m)}$.
\end{theorem}

\begin{proof}
Considering $G_G$ illustrated in Figure \ref{fig:WorstGlobal}, the optimal solution is the tour comprising all edges with weight $\frac{1}{m}$. We consider a typical run of the algorithm consisting of a phase of $T=Cm^3$ steps where $C$ is an appropriate constant. For the typical run we show the following:

\begin{enumerate}
\item A local optimum $\pi$ with $S(\pi)=0$ is reached with probability $1-e^{-\Omega(m)}$

\item The global optimal solution is not obtained with probability $1-m^{-\Omega(m)}$
\end{enumerate}
Then we state that only a direct jump from the local optimum to the
global optimum is possible, and the probability of this event is
$O(m^{-m/2})$.

First we show that with high probability $S(\pi_{init})\leq \varepsilon m$ holds for the initial solution $\pi_{init}$, where $\varepsilon$ is a small positive constant.

We count the number of permutations in which at least $\varepsilon m$,
$\varepsilon>0$ a small constant, of cluster-neighbourhoods are
correct. 

We should select $\varepsilon m$ of the clusters to be
followed by their specific neighbour, and consider the number of
different permutations of $m-\varepsilon m$ clusters:
\begin{equation}
\binom{m}{\varepsilon m} (m-\varepsilon m)!\label{equ:NoOfSolutions} 
\end{equation}

Some solutions are double-counted in this expression, so the actual
number of different solutions with $S(\pi)\geq \varepsilon m$ is less
than (\ref{equ:NoOfSolutions}). Therefore, the probability of having
more than $\varepsilon m$ clusters followed by their specific cluster,
is at most

$$\binom{m}{\varepsilon m} \frac{(m-\varepsilon m)!}{m!} =((\varepsilon m)!)^{-1} = O\left( \left(\frac{\varepsilon m}{2} \right)^{-\frac{\varepsilon m}{2}} \right) $$ 

Hence, with probability $1-O((\frac{\varepsilon m}{2})^{-\frac{\varepsilon m}{2}})$, $S(\pi_{init})\leq \varepsilon m$ holds and the initial solution has at at most $\varepsilon m$ correctly ordered clusters.

Now we analyze the expected time to reach a solution $\pi$ with
$S(\pi)=0$. The probability of a good ordering to change to a bad one
is at least

$$\left(\frac{1}{e} \right) \cdot \left(\frac{k}{m} \right) \cdot \left(\frac{m^2-m}{m^2} \right)^k$$
where $k$ is the number of edges which can be changed in each
operation. For \textit{jump} operation $k$ equals $3$. For all $m>2$,
it holds that $m<\frac{m^2}{2}$, so the probability above is at least

$$ \left(\frac{1}{e} \right)  \cdot \left(\frac{3}{m} \right ) \cdot \left(\frac{1}{2} \right)^3=\Omega(m^{-1})$$

Therefore, the expected time for each edge to be replaced with a bad edge is in $O(m)$ and for $m$ edges it is in $O(m^2)$. 

Now we consider a phase of $T=Cm^3$ iterations and show that the local optimum is reached with high probability.

Let $C=2C'$ and consider a phase of $2C'm^2$ iterations while assuming that the local optimum is expected to be reached in time $C'm^2$. Then by means of Markov's Inequality we have
$$\Pr(T'>2C'm^2)\leq \frac{1}{2}.$$

Repeating this $m$ times, the probability of not reaching the local
optimum is $2^{-m}$. Therefore, the algorithm reaches the
local optimum with probability
$1-2^{-m}=1-e^{-\Omega(m)}$ during the phase of
$T=Cm^3$ steps.

To prove that with high probability, the global optimum is not reached
during the considered phase, note first that by Lemma \ref{lemma:localopt},
any jump to a solution closer to the optimum other than directly to the Global Optimum will be rejected.
  
Furthermore, for the initial
solution $S(\pi_{init})\leq \varepsilon m$. Therefore, only
non-optimal solutions $\pi$ with $S(\pi)\leq \varepsilon m$ are
accepted by the algorithm. In order to obtain an optimal solution the
algorithm has to produce the optimal solution from a solution $\pi$
with $S(\pi)\leq \varepsilon m$ in a single mutation step.  We now
upper bound the probability of such a direct jump which changes at
least $(1-\varepsilon)m$ clusters to their correct order. Such
a move needs at least $\frac{(1-\varepsilon)m}{3}$ operations in the
same iteration. Taking into account that these \emph{Jump} operations
may be acceptable in any order, the probability of a direct jump is at
most
\begin{equation}
\frac{1}{e \left(\frac{(1-\varepsilon)m}{3} \right)!} .\frac{1}{m^{\frac{(1-\varepsilon)m}{2}}} \cdot \left(\frac{(1-\varepsilon)m}{3} \right)! = m^{-\Omega(m)}. \label{equ:withepsilon}
\end{equation}

So in a phase of $O(m^3)$ iterations the probability of having such a direct jump is by union bound at most $m^{-\Omega(m)+3} = m^{-\Omega(m)}$.


So far we have shown that a local optimum $\pi$ with $S(\pi)=0$ is reached with probability $1 - e^{-\Omega(m)}$ within the first $T=Cm^3$ iterations.

The probability of obtaining an optimal solution from a solution $\pi$ with $S(\pi)=0$ is at most
$$\frac{1}{e \left(\frac{m}{3} \right)!} \cdot\frac{1}{m^{\frac{m}{2}}} \cdot \left(\frac{m}{3} \right)!=e^{-1} \cdot m^{-\frac{m}{2}}$$

We now consider an additional phase of $(\frac{m}{2})^\frac{m}{2}$ steps after having obtained a local optimum. Using the union bound,  the probability of reaching the global optimum in this phase is at most

$$\left(\frac{m}{2} \right)^\frac{m}{2} \cdot e^{-1} \cdot m^{-\frac{m}{2}} \leq \left(\frac{1}{2}\right)^{\frac{m}{2}}.$$ 
 
 As a result, the probability of not reaching the optimal solution in these $(\frac{m}{2})^\frac{m}{2}$ iterations is $1-2^{-\frac{m}{2}}= 1-e^{-\Omega(m)}$. Altogether, the optimization time is at least $(\frac{m}{2})^\frac{m}{2}$ with probability $1-e^{-\Omega(m)}$.
\end{proof}

\section{Discussion of Generalisations}
\label{sec:gen}

The problems we have examined in this work are bilevel optimisation
problems where the upper level problem, namely the \emph{leader}, and
the lower level problem, the \emph{follower}, shares an objective
function. The general bilevel optimisation problem also includes the
setting where the leader and the follower have different
objectives. Given the decision of the leader, the follower makes a
decision according to his objective function which might be
conflicting with the objective function of the leader. An example of
such a problem is where the leader places toll booths across a road
network and the followers try to find the cheapest way from a point
\emph{A} to a point \emph{B} by finding a path that avoids as many
toll booths as possible. Here, the leader can only learn the objective
function value of its decision after the follower picks the optimum
path. Unlike the GMSTP and GTSP, the objective functions of upper and
lower level problems are conflicting in this toll booth problem.

For a given solution visited in the upper level problem, the
evaluation cost is, in the worst case, the computational complexity of
the lower level problem. If the lower level problem can be solved in
polynomial time, then a fixed-parameter bound on the the size of the
upper level solution is sufficient for a fixed-parameter tractable
problem. For when the upper level solution is bounded by a parameter
\emph{k} of the original problem, any global random search heuristic
on the upper level problem will be able to find the optimal upper
level solution in no more than $f(k)$ iterations for some function
$f(k)$ and will make $f(k)\cdot \poly(n)$ basic operations in total.

In our case, the \emph{Global structure representation} of GMSTP and
GTSP, the size of an upper level solution is bounded above by $m^2$
since it is enough to indicate whether any two clusters are connected
or not to precisely define a solution. On the other hand the
spanned-nodes representation of GMSTP needs a size of $m\log(n)$ to
represent \emph{which node} is selected in \emph{each
  cluster}. If the solution size is restricted by a parameter $m$,
uniform random search on the bitstring of length $O(f(m))$ will
find the optimal solution in $2^{O(f(m))}$ iterations in expectation.
With \emph{Global structure representation}, if we pick our solutions
uniformly at random the probability of picking a unique optimal
solution is $(1/2)^{m^2}$ which will occur in $O(m^{2m})$ time in
expectation while uniform random search with the spanned node
representation takes $\Omega (n^{m})$ trials in expectation.

\section*{Conclusions}

Evolutionary bilevel optimization has gained an increasing interest in recent years. With this article we have contributed to the theoretical understanding by considering two classical NP-hard combinatorial optimization problems, namely the generalized minimum spanning tree problem and the generalized traveling salesperson problem.
We studied evolutionary algorithms for the mentioned problems in the parameterized setting. 
Using parameterised computational complexity analysis of evolutionary algorithms for the generalized minimum spanning tree problem,
we have examined two representations for the upper layer solutions and
their corresponding deterministic algorithms for the lower layer. Our
results show that the \emph{Global Structure Representation} leads to
fixed parameter evolutionary algorithms. By presenting hard instances
for each of the two approaches, we have pointed out where they run
into difficulties. Furthermore, we have shown that the two
representations for the generalized minimum spanning tree problem are highly complementary by proving that they are
highly efficient on the hard instance of the other algorithm. After having achieved these results for the generalized minimum spanning tree problem, we turned our attending to the generalized traveling salesperson problem. We showed that using the global structure representation leads to fixed parameter evolutionary algorithms with respect to the number of clusters. Furthermore, we pointed out a worst case instance where the optimization time grows exponential with respect to the number of clusters and discussed generalizations of the results.

\bibliographystyle{apalike}
\bibliography{myreferences}

\begin{thebibliography}{}

\bibitem[Auger and Doerr, 2011]{BookAugDoe}
Auger, A. and Doerr, B., editors (2011).
\newblock {\em Theory of Randomized Search Heuristics: Foundations and Recent
  Developments}.
\newblock World Scientific.

\bibitem[Corus et~al., 2013]{Corus2013GMST}
Corus, D., Lehre, P.~K., and Neumann, F. (2013).
\newblock The generalized minimum spanning tree problem: a parameterized
  complexity analysis of bi-level optimisation.
\newblock In Blum, C. and Alba, E., editors, {\em GECCO}, pages 519--526. ACM.

\bibitem[Deb and Sinha, 2009]{DBLP:conf/emo/DebS09}
Deb, K. and Sinha, A. (2009).
\newblock Solving bilevel multi-objective optimization problems using
  evolutionary algorithms.
\newblock In Ehrgott, M., Fonseca, C.~M., Gandibleux, X., Hao, J.-K., and
  Sevaux, M., editors, {\em EMO}, volume 5467 of {\em Lecture Notes in Computer
  Science}, pages 110--124. Springer.

\bibitem[Deb and Sinha, 2010]{DBLP:journals/ec/DebS10}
Deb, K. and Sinha, A. (2010).
\newblock An efficient and accurate solution methodology for bilevel
  multi-objective programming problems using a hybrid evolutionary-local-search
  algorithm.
\newblock {\em Evolutionary Computation}, 18(3):403--449.

\bibitem[Downey and Fellows, 1999]{Downey1999}
Downey, R.~G. and Fellows, M.~R. (1999).
\newblock {\em Parameterized Complexity}.
\newblock Springer-Verlag.
\newblock 530 pp.

\bibitem[Fischetti et~al., 1997]{ClusterOpt1997}
Fischetti, M., Salazar~Gonz{\'a}lez, J.~J., and Toth, P. (1997).
\newblock A branch-and-cut algorithm for the symmetric generalized traveling
  salesman problem.
\newblock {\em Operations Research}, 45(3):378--394.

\bibitem[Hu and Raidl, 2011]{HuR11}
Hu, B. and Raidl, G.~R. (2011).
\newblock An evolutionary algorithm with solution archive for the generalized
  minimum spanning tree problem.
\newblock In Moreno-D\'{\i}az, R., Pichler, F., and Quesada-Arencibia, A.,
  editors, {\em EUROCAST (1)}, volume 6927 of {\em Lecture Notes in Computer
  Science}, pages 287--294. Springer.

\bibitem[Hu and Raidl, 2012]{HuR12}
Hu, B. and Raidl, G.~R. (2012).
\newblock An evolutionary algorithm with solution archives and bounding
  extension for the generalized minimum spanning tree problem.
\newblock In Soule, T. and Moore, J.~H., editors, {\em GECCO}, pages 393--400.
  ACM.

\bibitem[Koh, 2007]{Koh07}
Koh, A. (2007).
\newblock Solving transportation bi-level programs with differential evolution.
\newblock In {\em IEEE Congress on Evolutionary Computation}, pages 2243--2250.
  IEEE.

\bibitem[Kratsch et~al., 2010]{DBLP:conf/ppsn/KratschLNO10}
Kratsch, S., Lehre, P.~K., Neumann, F., and Oliveto, P.~S. (2010).
\newblock Fixed parameter evolutionary algorithms and maximum leaf spanning
  trees: A matter of mutation.
\newblock In {\em Proceedings of the Eleventh Conference on Parallel Problem
  Solving from Nature}, pages 204--213.

\bibitem[Kratsch and Neumann, 2013]{KratschNeumannGECCO09}
Kratsch, S. and Neumann, F. (2013).
\newblock Fixed-parameter evolutionary algorithms and the vertex cover problem.
\newblock {\em Algorithmica}, 65(4):754--771.

\bibitem[Legillon et~al., 2012]{LegillonLT12}
Legillon, F., Liefooghe, A., and Talbi, E.-G. (2012).
\newblock Cobra: A cooperative coevolutionary algorithm for bi-level
  optimization.
\newblock In {\em IEEE Congress on Evolutionary Computation}, pages 1--8. IEEE.

\bibitem[Motwani and Raghavan, 1995]{motwani:randomized}
Motwani, R. and Raghavan, P. (1995).
\newblock {\em Randomized Algorithms}.
\newblock Cambridge University Press.

\bibitem[Myung et~al., 1995]{MyungLT95}
Myung, Y.-S., ho~Lee, C., and wan Tcha, D. (1995).
\newblock On the generalized minimum spanning tree problem.
\newblock {\em Networks}, 26(4):231--241.

\bibitem[Neumann and Witt, 2010]{Neumann2010}
Neumann, F. and Witt, C. (2010).
\newblock {\em Bioinspired Computation in Combinatorial Optimization:Algorithms
  and Their Computational Complexity}.
\newblock Springer-Verlag New York, Inc., New York, NY, USA, 1st edition.

\bibitem[Pop, 2004]{Pop04}
Pop, P.~C. (2004).
\newblock New models of the generalized minimum spanning tree problem.
\newblock {\em J. Math. Model. Algorithms}, 3(2):153--166.

\bibitem[Sutton and Neumann, 2012a]{SuttonN12}
Sutton, A.~M. and Neumann, F. (2012a).
\newblock A parameterized runtime analysis of evolutionary algorithms for the
  euclidean traveling salesperson problem.
\newblock In Hoffmann, J. and Selman, B., editors, {\em AAAI}. AAAI Press.
\newblock Extended technical report available at
  http://arxiv.org/abs/1207.0578.

\bibitem[Sutton and Neumann, 2012b]{Sutton2012makespan}
Sutton, A.~M. and Neumann, F. (2012b).
\newblock A parameterized runtime analysis of simple evolutionary algorithms
  for makespan scheduling.
\newblock In {\em Proceedings of the Twelfth Conference on Parallel Problem
  Solving from Nature (PPSN 2012)}, pages 52--61. Springer.

\end{thebibliography}

\end{document}